\documentclass{article}
\usepackage{ijcai24}

\usepackage{times}
\usepackage{soul}
\usepackage{url, natbib}
\usepackage[hidelinks]{hyperref}
\usepackage[utf8]{inputenc}
\usepackage[small]{caption}
\usepackage{graphicx, bm}
\usepackage{amsmath, amssymb}
\usepackage{amsthm}
\usepackage{booktabs}
\usepackage{algorithm}
\usepackage{algorithmic, enumitem}
\usepackage[switch]{lineno}

\newtheorem{theorem}{Theorem}[section]

\theoremstyle{plain}

\theoremstyle{definition}

\newtheorem{problem}[theorem]{Problem}
\newtheorem{remark}[theorem]{Remark}

\title{A View on Out-of-Distribution Identification from\\ a Statistical Testing Theory Perspective}

\author{
    Author Name
    \affiliations
    Affiliation
    \emails
    email@example.com
}

\author{
Alberto Caron$^1$
\and
Chris Hicks$^1$\and
Vasilios Mavroudis$^1$
\affiliations
$^1$The Alan Turing Institute\\
\emails 
\{acaron, c.hicks, vmavroudis\}@turing.ac.uk,
}

\begin{document}

\maketitle

\begin{abstract}
    We study the problem of efficiently detecting Out-of-Distribution (OOD) samples at test time in supervised and unsupervised learning contexts. While ML models are typically trained under the assumption that training and test data stem from the same distribution, this is often not the case in realistic settings, thus reliably detecting distribution shifts is crucial at deployment. We re-formulate the OOD problem under the lenses of statistical testing and then discuss conditions that render the OOD problem identifiable in statistical terms. Building on this framework, we study convergence guarantees of an OOD test based on the Wasserstein distance, and provide a simple empirical evaluation.
\end{abstract}

\section{Introduction}

Supervised and unsupervised learning models traditionally operate under the assumption that the distributions of data observed during training and testing are the same \citep{murphy2012machine}. However, this ideal scenario rarely holds true in real-world applications, where data distributions can shift. OOD detection generally refers to the task of identifying instances that lie outside the distribution of data seen during training \citep{yang2021generalized, liu2021towards_survey}. It serves as a crucial component towards safe and trustworthy AI when deploying machine learning models in real-world scenarios such as self-driving cars (e.g., \citep{ji2021poltergeist}), medical diagnostics and cybersecurity \citep{hendrycks2021unsolved}. Failure to accurately detect OOD samples in these settings can lead to severe consequences.


There is a very rich literature on methods to detect OOD. Some highly influential contributions include: \cite{hendrycks2016baseline, liang2018enhancing, sun2021react}, that propose a softmax score for neural networks; \cite{ren2019likelihood} that advocates for likelihood ratios based scores; \cite{liu2020energy} that instead utilize energy-based scores; and \cite{huang2021importance} that uses gradient information of the KL divergence. Theoretical studies on the problem of OOD detection are limited. Recent works include \cite{ye2021towards, fort2021exploring} and \cite{morteza2022provable}. In particular, \cite{zhang2021understanding} analyzes the causes of failures in OOD detection with generative models, while \cite{fang2022out} derive conditions for learnability of the OOD problem under a PAC-learnability framework.

In this work, we adopt a statistical testing approach to the problem \citep{lehmann1986testing}. While this perspective is not entirely new \citep{zhang2021understanding, haroush2022a}, we progress a step further and derive theoretical guarantees under a non-parametric statistical testing framework based on the Wasserstein distance \citep{vallender1974calculation, panaretos2019statistical}. These results help understand the conditions for identifiability of the OOD classification problem.

\section{Problem Framework}

Suppose we are in a supervised learning setting and we observe i.i.d.~data $\mathcal{D}_n = \{(X_i, Y_i)\}^n_{i=1}$ at training time, stemming from the joint distribution $(X_i, Y_i) \sim P_{(X_{in}, Y_{in})} \in \mathcal{P}_{in}$ (In-Distribution, or ID). At test time, data might generate either from the same joint distribution as training $P_{(X_{in}, Y_{in})} \in \mathcal{P}_{in}$, or a fixed out-of-training (OOD) joint distribution $P_{(X_{out}, Y_{out})} \in \mathcal{P}_{out}$. Thus, at test time, data points are sampled from the mixture
\begin{equation*}
    P_{(X, Y)} = (1 - \delta) P_{(X_{in}, Y_{in})} + \delta  P_{(X_{out}, Y_{out})}~,
\end{equation*}
where $\delta \in [0, 1)$ indicates the prevalence of OOD samples. In this supervised learning context, a distribution shift in the joint $P_{(X, Y)} = P(Y|X) P(X) = P(X|Y) P(Y)$ can stem from either a shift in the marginal $P(X)$, also known as `covariate shift', or in the marginal $P(Y)$, known as `semantic shift', or in both at the same time \citep{yang2021generalized}. In unsupervised settings shifts can occur only in $P(X)$. In typical supervised learning settings however, at test time, we do not observe output labels $Y_i$. Yet, this is not always the case. For example in online learning we observe a stream of independent $(X_i,Y_i)$ pairs that we can use to update the model, thus we can possibly detect both anomalies in $P(X)$ and in $P(Y)$. Finally, in time series we observe a stream of temporally dependent $(Y_t, Y_{t+1},...)$ data points, where the objective is to detect a change point in time $t$ \citep{saatcci2010gaussian, aminikhanghahi2017survey, van2020evaluation}. Given these different types of data streaming, we define the generalized OOD detection problem as follows 
\begin{problem}[OOD Detection] \label{def:OOD}
    Given an ID distribution $P_{\mathcal{D}^{in}} \in \mathcal{P}_{in}$ and training data $\mathcal{D}^{in}_n \sim P_{\mathcal{D}^{in}}$, the aim is to construct a binary classifier $h^{out} \in \mathcal{H}^{out} \subset \{h : \mathcal{D} \rightarrow \{0,1\}\}$ such that
    \begin{equation*}
    h^{out} (\mathcal{D}_i) = 
        \begin{cases}
            0 \quad \text{ if } ~~ \mathcal{D}_i \sim P_{test} = P_{\mathcal{D}^{in}} \\
            1 \quad \text{ if } ~~ \mathcal{D}_i \sim P_{test} = P_{\mathcal{D}^{out}}
        \end{cases},
\end{equation*}
where $\mathcal{D}_i$ denotes the $i$-th entry of the test dataset.
\end{problem}
The analogy between the `zero-shot' OOD classification problem above and statistical tests has been drawn by recent works \citep{zhang2021understanding, haroush2022a}. In statistical testing theory \citep{lehmann1986testing}, the goal is to test a null versus an alternative hypothesis $H_0: \theta \in \Theta_0$ VS $H_1: \theta \in \Theta_1$, via a test indicator function $\phi_{\mathcal{R}} (T(z)) = \mathbb{I} [T(z) \in \mathcal{R}]$, where $z \in Z$ is a random variable, $T(z)$ is a test statistic and $\mathcal{R}$ an acceptance region. Also, we denote the power of a test, or True Positive Rate (TPR), as the probability of accepting $H_1$ when this is true: $\varrho (T) = p (T \in \mathcal{R} \mid \theta \in \Theta_1)$,. We can then draw the following equivalence between OOD detection and a `goodness-of-fit' tests:
\begin{remark}[OOD Test]
    Given the hypotheses $H_0: P_{test} = P_{\mathcal{D}^{in}} ~vs~ H_1: P_{test} \neq P_{\mathcal{D}^{in}}$, Problem \ref{def:OOD} is a statistical testing problem where $\phi_{\mathcal{R}} (\mathcal{D}_i) = \mathbb{I} [h^{out}(\mathcal{D}_i) > \lambda_{1-\alpha}]$, with test statistic $T(\cdot) = h^{out}(\cdot)$ and critical value $\lambda_{1-\alpha}$.
\end{remark}
The critical test value $\lambda_{1-\alpha}$ is usually cross-validated using the training data, and $\alpha$ indicates the type-I error probability of the test (typical values are $\{0.10, 0.05, 0.01\}$). The task is to design a test statistic $T(\cdot) = h^{out}(\cdot)$ that gives the best guarantees in terms of test power $\varrho (T) = p (T \in \mathcal{R} \mid P_{test} \neq P_{\mathcal{D}^{in}})$ (equivalently known as False Positive Rate).

\section{A Wasserstein Distance OOD Test}

The test statistic we consider is based on the Wasserstein distance. The $p$-Wasserstein distance is defined, given two probability measures $P$ and $Q$ on $\mathbb{R}^d$, as
\begin{equation*}
    W_p(P, Q)=\inf _{\gamma \in \Gamma(P,Q)} \Big( \mathbb{E}_{(Z,V) \sim \gamma} \| Z - V \|^p \Big)^{1 / p} ~ ,
\end{equation*}
where $\Gamma(P,Q)$ is the space of all couplings of $P$ and $Q$ (joint probability measures whose marginals are $P$ and $Q$). The Wasserstein distance, unlike other distributional divergences, possesses typical metric properties, such as symmetry and triangle inequality \citep{ramdas2017wasserstein, panaretos2019statistical}, that we need in order to derive the results in the next section. The Wasserstein OOD test is sketched out as follows. A model of the data generating distribution $P_{\theta}$ (e.g., regression, classification, generative, etc.), parametrized by $\theta \in \Theta$, is learnt using $\mathcal{D}_n \sim P_{train}$. Then, at test time, the test statistic for the hypotheses $H_0: D_i \sim P_{\theta}$ ~vs~ $H_1: D_i \sim Q \neq P_{\theta}$ is defined as $T^{wass}(\mathcal{D}_i) = W_p (P_{\theta}, Q)$, replacing $W_p (\cdot, \cdot)$ with its sample equivalent, and the critical test value $\lambda_{1-\alpha}$ is computed as the on the training data.

\subsection{OOD Identifiability and Test Power} \label{sec:OODident}

In this section we derive theoretical properties of the Wasserstein OOD test, and we use these to shed light on the identifiability, or learnability \citep{fang2022out}, of the OOD detection problem more in general. We start by demonstrating that, in the limit of OOD samples $ \rightarrow \infty$, the Wasserstein OOD test is uniformly consistent, i.e., its power (or True Positive Rate) $\rightarrow 1$, under some specific condition linking the ID $P_{\theta}$ and the OOD $Q$ distributions. Let $m$ be the OOD samples at test time, $\Delta_m$ a scalar depending only on $m$, and $Q_m$ the OOD distribution corresponding to a certain $m$, then we have:
\begin{theorem}[Uniform Consistency] \label{thm:consist}
    Let $\mathcal{D}_m$ be a test dataset. The test based on $T^{wass}_m = m^{1/2} W_p (P_{\theta}, Q)$ for hypotheses $H_0: D_m \sim P_{\theta}$ ~vs~ $H_1: D_m \sim Q \neq P_{\theta}$, is such that
    $$ \varrho (T^{wass}_m) = p(T^{wass}_m > \lambda_{1-\alpha} \mid Q \neq P_{\theta} ) \rightarrow 1 ~ , $$
    as $m \rightarrow \infty$, over alternatives $Q_m$ that satisfy $n^{1/2} W(P_{\theta}, Q_m) \geq \Delta_m$, where $\lim_{m \rightarrow \infty} \Delta_m = \infty$.
\end{theorem}
Proof is provided in the appendix. The intuition behind Thm.~\ref{thm:consist} is that the test power (or TPR) of $T^{wass}_m$ is asymptotically optimal in the limit of $m \rightarrow \infty$ if the OOD $Q_m$ is far enough from the ID $P_{\theta}$. Thus, the OOD detection problem is `identifiable' when there is asymptotically no `overlap' between the ID and OOD distributions (note that this is valid for any test) \citep{ye2021towards, fang2022out}. Notice that while we prove consistency in terms of Wasserstein distance, the same result can in theory be obtained also for other distributional distances, for example Kolmogorov-Smirnov (KS) distance (but not the KL and JS divergence, which are not symmetric and do not satisfy the triangle inequality). However, as we will discuss later, Wasserstein distance also has practical implementation advantages; e.g., KS distance cannot easily be computed for multivariate distributions.

While uniform consistency holds in the limit of $m \rightarrow \infty$, we can also derive a non-asymptotic lower bound on $T^{wass}_m$ using concentration inequalities as follows. Notice that, contrary to Thm.~\ref{thm:consist}, the following results pertain uniquely to the Wasserstein distance case.
\begin{theorem}[Non-Asymptotic Lower Bound]  \label{thm:3.2}
    Let $\mathcal{D}_m$ be a test dataset. The test $T^{wass}_m = m^{1/2} W_p (P_{\theta}, Q)$ for hypotheses $H_0: D_m \sim P_{\theta}$ ~vs~ $H_1: D_m \sim Q \neq P_{\theta}$, is such that
    \begin{align*}
        \varrho (T^{wass}_m) & = p(T^{wass}_m > \lambda_{1-\alpha} \mid Q \neq P_{\theta} ) \geq \\
        & \geq 1 - \exp \left\{  - \gamma_p \frac{\phi'}{2} \big( \Delta_m - \lambda_{n, 1- \alpha} \big)^2  \right\} ~,
    \end{align*}
    if $m^{1/2} W_p (P_{\theta} , Q_m ) \geq \Delta_m$ and $\Delta_m \geq \lambda_{m, 1 - \alpha}$.
\end{theorem}
Here, $\gamma_p$ and $\phi'$ are defined according to the appendix. What happends then if condition $m^{1/2} W_p (P_{\theta} , Q_m ) \geq \Delta_m$ does not hold, i.e., when $P_{\theta}$ and $Q_m$ are very near (`near OOD')? We can derive a worst case upper bound of the distance based tests as follows.
\begin{theorem}[Worst Case Upper Bound] \label{thm:wcupper}
    Let $\mathcal{D}_m$ be a test dataset. The test $T^{wass}_m = m^{1/2} W_p (P_{\theta}, Q)$ for hypotheses $H_0: D_m \sim P_{\theta}$ ~vs~ $H_1: D_m \sim Q \neq P_{\theta}$, is such that
    $$ \sup \varrho (T^{wass}_m) = p(T^{wass}_m > \lambda_{1-\alpha} \mid Q \neq P_{\theta} ) \leq \alpha $$
    as $m \rightarrow \infty$, for alternatives $Q_m$ satisfying $m^{1/2} W_p (P_{\theta}, Q_m) \rightarrow 0$.
\end{theorem}
This means that the OOD detection problem is not identifiable for OOD distrubutions such that $m^{1/2} W_p (P_{\theta}, Q_m) \rightarrow 0$. In fact, no test can have high power in such cases where the OOD distribution is at a distance $o(m^{1/2})$ from $P_{\theta}$ \citep{lehmann1986testing}. Lastly, we derive an asymptotic upper bound for the `intermediate' case of $W(P_{\theta}, Q)$ not diverging nor converging to 0, but of distance converging to a constant $\delta$, such as the one depicted in the MNIST example in Section \ref{sec:MNIST}. This reads:
\begin{theorem}[Intermediate Case Asymptotic Upper Bound]  \label{thm:interbound}
    Let $\mathcal{D}_m$ be a test dataset. The test $T^{wass}_m = m^{1/2} W_p (P_{\theta}, Q)$ for hypotheses $H_0: D_m \sim P_{\theta}$ ~vs~ $H_1: D_m \sim Q \neq P_{\theta}$, is such that
    \begin{align*}
        \varrho (T^{wass}_m) & = p(T^{wass}_m > \lambda_{1-\alpha} \mid Q \neq P_{\theta} ) \leq \\
        & \leq \exp \left\{  - \gamma_p \frac{\phi'}{2} \big( \lambda_{n, 1- \alpha} - \delta \big)^2  \right\} ~ ,
    \end{align*}
    for alternatives $Q_m$ such that $m^{1/2} W_p (P_{\theta}, Q_m ) \rightarrow \delta < \lambda_{n, 1 - \alpha}$.
\end{theorem}
In general, for the case where $m^{1/2} W_p (P_{\theta}, Q_m ) \rightarrow \delta$ and $\delta \in (0, \infty)$, the test power will converge to a limit strictly between $\alpha$ and $1$.

\subsection{Distributional Distance Tests}

The theoretical results in the previous section shed light on the identifiability of the OOD detection problem, by laying out the conditions that guarantee certain upper bounds or convergence in terms of test power, or TPR, $\rho (T^{wass}_m)$. Although these results pertain to the proposed Wasserstein distance OOD test, they can be in theory obtained also for other distributional distances (e.g., total variation distance, KS distance, etc.). In this section we explain why distributional distances are preferable, compared, e.g., to entropy and k-NN based tests \citep{ren2019likelihood} and detail the advantages of using Wasserstein distance specifically.

\vspace{0.2cm} \hspace{0.2cm} \emph{Q1: Why Distance Based OOD Tests?} \vspace{0.2cm}

\begin{figure}
    \centering
    \includegraphics[scale=0.35]{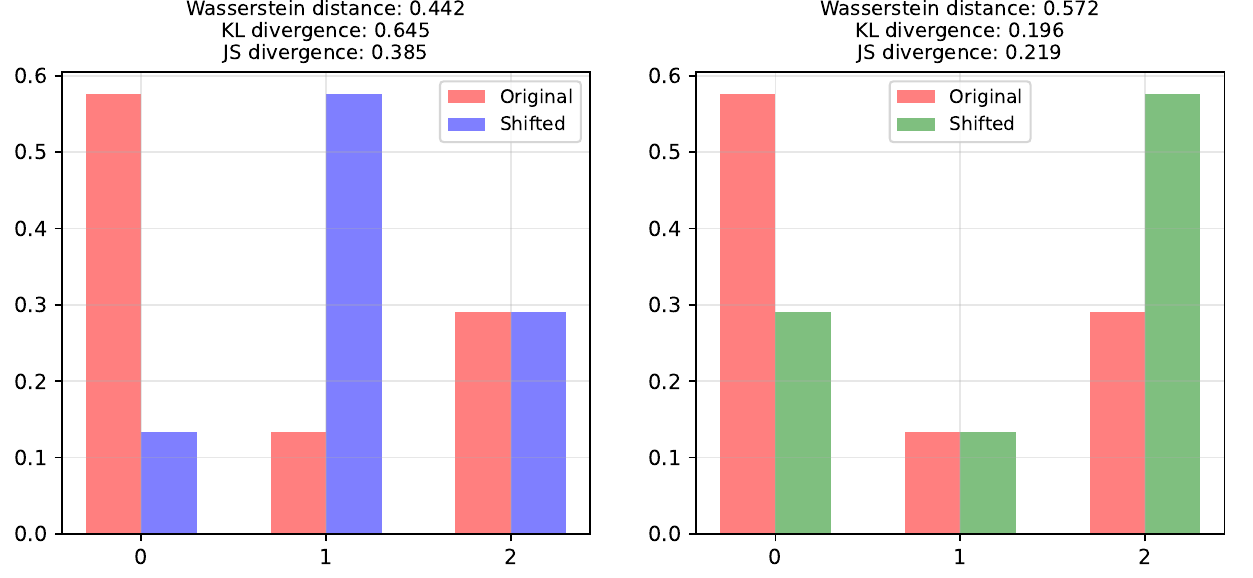}
    \caption{Examples of a discrete distribution shifts where KL and JS divergences offer a less informative measure, while $W(P,Q)$ is able to capture that the shift on the right is geometrically much further apart from the reference distribution than the one on the left.}
    \label{fig:discrete_shift}
\end{figure}

\noindent OOD metrics based on distributional distances guarantee that aleatoric (pure noise) components are removed from consideration. To prove this point, consider the Shannon entropy \citep{shannon1948mathematical} of a distribution $P_x$ over a random variable $X$, $\mathbb{H} [P_x] = - \mathbb{E}_{x \sim P_x} \log P (x)$, as a measure of Total Variation (TV) of $X$. TV can be decomposed w.r.t.~another random variable $Y \sim Q_y$ into the following \citep{cover1999elements}: 
\begin{equation*}
    \underbrace{\mathbb{H} [P_x]}_{\text{total variation}} = \underbrace{\mathbb{H} [P_x \mid Q_y]}_{\text{aleatoric variation}} + \underbrace{I [P_x; Q_y]}_{\text{epistemic variation}} .
\end{equation*}
The conditional entropy $\mathbb{H} [P_x \mid Q_y]$ is a natural measure of aleatoric uncertainty or pure noise, that does not bear any `semantic' information (e.g., an image background). Conversely, the mutual information $I [P_x; Q_y] = \mathbb{H} [P_x] - \mathbb{H} [P_x \mid Q_y] = d_{KL} (P_x, Q_y) $, where $d_{KL} (\cdot, \cdot)$ is the KL divergence \citep{lindley1956infogain}, is a measure of epistemic uncertainty only, as it represents how much information can be gained on $X$ by observing realizations of $Y$ (and viceversa), net of aleatoric components. In an OOD context this means intuitively that the smaller $I [P_x; Q_y]$, the more likely it is that we are observing data from the same distribution, $P_x = Q_y$. Finally, notice that there is a direct link between KL-divergence and Wasserstein distance (see Appendix \ref{sec:appen:KLwass}), such that both are infact different measures of mutual information $I[\cdot;\cdot]$.

\vspace{0.2cm} \hspace{0.2cm} \emph{Q2: Why Wasserstein Distance?} \vspace{0.2cm}

\begin{figure}
    \centering
    \includegraphics[scale=0.32]{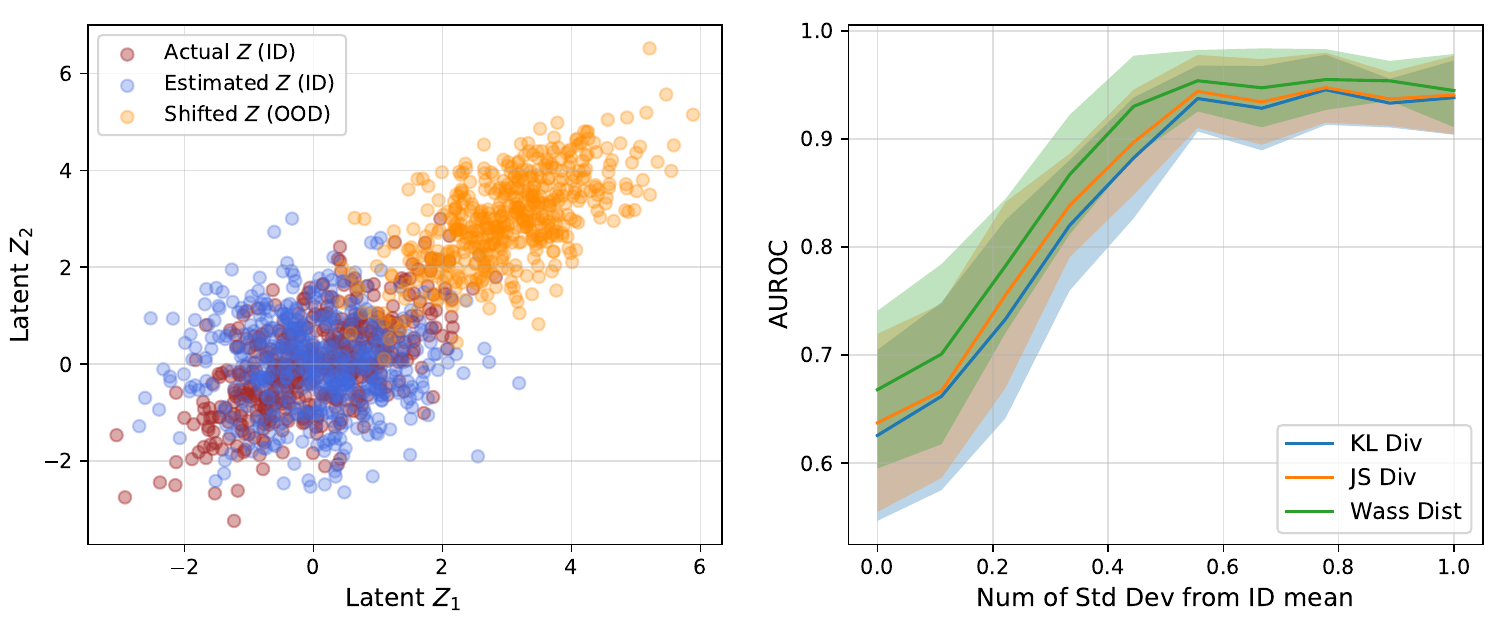}
    \caption{The plot on the left depicts the latent factors $(Z_1, Z_2)$ true distribution, learnt distribution via FL, and OOD one. The plot on the right reports the mean AUROC of each OOD tests (with 90\% error bands) for number of standard deviations from the ID mean.}
    \label{fig:FAex}
\end{figure}

Inspired by optimal transport theory, the Wasserstein distance has clear geometric interpretations, and it embeds information about the geometry of the support $\mathcal{X}$ \citep{panaretos2019statistical}; see Figure \ref{fig:discrete_shift} above for an example. Thus, compared to the popular KL and JS divergences, Wasserstein can provide meaningful and smooth representations of the distance between two distributions $P$ and $Q$, even when there is no overlap between them \citep{arjovsky2017wasserstein}; see example in Section 6.2 of \cite{weng2019gan}. Theoretical results in Section \ref{sec:OODident} thus are not derivable for the KL divergence case, as this is not symmetric and does not satisfy the triangle inequality. Compared to KS distance (used in KS tests) instead, Wasserstein distance can be easily applied to multivariate data, whereas the former requires multivariate empirical cumulative distribution functions, which are notoriously hard to compute \citep{justel1997multivariate, langrene2021fast}. Then finally, contrary to parametric families of distances (e.g., Mahalanobis distance, even though this is technically a distance between a point and a distribution), it does not necessarily need parametric assumptions, although these can simplify computations as we show in the section below.

\vspace{0.1cm}

\begin{figure*}
    \centering
    \includegraphics[scale=1.2]{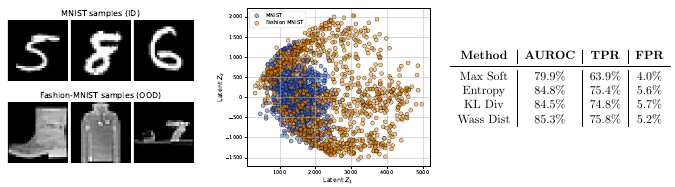}
    \caption{First plot on the left shows samples from MNIST (ID) and Fashion MNIST (OOD) datasets. Centre plot shows the distribution of the two principal latent factors, computed with Truncated SVD, on MNIST and Fashion MNIST. Table on the right reports results in terms of AUROC, TPR and FPR of the four OOD test considered.}
    \label{fig:MNIST}
\end{figure*}

Eventually, we stress that in the case of (multivariate) Gaussian distributions, the computation of distances mentioned above can be heavily simplified to closed form operation. In Appendix \ref{sec:simpli}, we include the simplified version of KL Divergence and 2-Wasserstein Distance for the case of $P$ and $Q$ being multivariate Gaussians.

\section{Experiments}

In this section we present results from two experiments: the first is a simple depiction of how test power changes depending on overlap, while the second implements some of the OOD tests on a image classification problem\footnote{Code is provided at \url{https://github.com/albicaron/OOD_test}.}.

\subsection{Generative Model Example} \label{sec:genmod}

Suppose we access a high-dimensional dataset at training time consisting of continuous $\mathcal{D}^n = \{\bm{x}_i\}^n_{i=1}$, where $\bm{X} \in \mathcal{X} = \mathbb{R}^d$. We set $n=500$ and $d=100$ as training sample and feature dimensions. We want to learn a simple generative model for $\{\bm{x}_i\}^n_{i=1}$, by assuming that these are generated by lower dimensional features $\{\bm{z}_i\}^n_{i=1}$ with $\bm{Z}_i \in \mathbb{R}^p$ where $p \ll d$. For visualization purposes, we pick $p=2$, generated randomly as a $\mathcal{N} (\bm{0}, \Sigma)$, where $\Sigma = [1.0 ~~ 0.6, ~ 0.6 ~~ 1.0]$. We learn $\{\bm{z}_i\}^n_{i=1}$ via a Factor Loading (FL) model \citep{gorsuch2014factor}, where $x_i = \mu + W z_i + \varepsilon$ and the latent factors $z_i$ are unobserved and assumed to be $\bm{h} \sim \mathcal{N} (\bm{0}, I)$, the noise is $\varepsilon \sim \mathcal{N} (0, \Psi)$. At test time we assume we sequentially receive independent sample batches of size $m$, $\{\bm{x}_i\}^m_{i=1}$, and we have to detect whether a batch OOD w.r.t.~the training sample (this could be the case, e.g., in quality assurance/control problems, continuous authentication in security \citep{eberz2017evaluating} or experimental design) through their latent features $\{\bm{z}_i\}^m_{i=1}$. 

We construct an OOD test to assess whether a test batch is OOD as follows: \textit{i)} split the data $\{\bm{x}_i\}^n_{i=1}$ into 80\%-20\% train-validation; \textit{ii)} fit a FL model on the 80\% train set and compute the estimated latent factors $\{\hat{\bm{z}}_i\}^n_{i=1}$; \textit{iii)} use the 20\% validation set to compute the estimated $\{\hat{\bm{z}}^m_i\}$; \textit{iv)} Compute the KL, JS and Wasserstein distances between the FL estimated $P^{train}_{\bm{z}^n | \bm{x}} = p (\bm{z}^n | \bm{x})$ and $P^{test}_{\bm{z}^m | \bm{x}} = p (\bm{z}^m | \bm{x})$ --- notice that these densities are both $p$-dimensional Gaussians by model assumption, so we can use the computational simplifications of Appendix \ref{sec:simpli}; \textit{v)} Repeat the previous steps for $K$ folds and compute the OOD test critical value $\lambda_{1- \alpha}$ as the $1-\alpha$ quantile of each distance $d(P^{train}_{\bm{z}^n | \bm{x}}, P^{test}_{\bm{z}^m | \bm{x}})$. At test time we observe a sequence of 100 independent sample batches, of which 50\% are ID, while 50\% are OOD and exhibit a shift in the latent features $\{\bm{z}_i\}^m_{i=1}$ away from the ID mean, and equal to $\texttt{linspace}(0.0, 1.0, 10.0)$ times the standard deviation of the true train ID features $\{\bm{z}_i\}^n_{i=1}$. 

We report in Figure \ref{sec:exper} the following quantities: the first plot on the left depicts the true latent features $(Z_1, Z_2)$, the estimated features via FL, and the features subject to an OOD distributional shift; the plot on the right instead reports the Area Under the Receiver Operating Characteristic (AUROC) curve, computed between the true positive rates (test power) and the false positive rates (probability of type-I error $\alpha)$ for each method, for increasing values of OOD mean shift (in terms of number of standard deviations). Notice that AUROC reaches the optimal 95\% level for all the OOD tests, as the OOD distribution more clearly separates from the ID one and overlap reduces. 

\subsection{MNIST Classification} \label{sec:MNIST}

The second experiment involves an image classification task. We consider the MNIST dataset at training time, and learn a CNN probabilistic classifier $p(y_i | x_i)$ on the labels in $\{(x_i, y_i)\}^n_{i=1}$. At test time, we receive 50\% of samples from MNIST (ID) and 50\% samples from Fashion MNIST, which is treated as the OOD dataset, and we have to correctly detect which ones are OOD. Samples from each dataset are depicted in the first plot on the left in Figure \ref{fig:MNIST}. In the middle plot of Figure \ref{fig:MNIST}, we report the two principal latent features of each dataset, learnt via Truncated SVD. As can be seen their distributions appear to have a good degree of overlap, which is not surprising as MNIST and Fashion MNIST images share common patterns. The OOD tests are constructed similarly to Section \ref{sec:genmod}, using the softmax distributions $p(y_i|x_i)$. We consider a standard `max-softmax' OOD detector \citep{hendrycks2016baseline, liang2018enhancing}, an entropy based detector (that computes the entropy of the softmax distribution, and a KL and Wasserstein distance OOD detectors. For the KL and Wasserstein detectors, the distance is computed between $p(y_i|x_i)$ and the uniform distribution as a reference. Results in terms of AUROC, TPR and FPR are reported in the table on the right of Figure \ref{fig:MNIST}, and show that the Wasserstein based OOD test slightly outperforms the others.

\section{Conclusion} \label{sec:exper}

In this short paper we present novel theoretical results that shed light on the identifiability of OOD detection \citep{fang2022out}. These results stem from recasting the problem as one of statistical testing and leverage the properties of the Wasserstein distance to derive asymptotic and non-asymptotic bounds on test power. We conclude with two simple experiments on generative modelling and classification.

\section*{Acknowledgments}

Research funded by the Defence Science and Technology Laboratory (Dstl) which is an executive agency of the UK Ministry of Defence providing world class expertise and delivering cutting-edge science and technology for the benefit of the nation and allies. The research supports the Autonomous Resilient Cyber Defence (ARCD) project within the Dstl Cyber Defence Enhancement programme.

\nocite{*}
\bibliographystyle{named}
\bibliography{ijcai24}

\clearpage

\appendix

\section{Proofs of Theorems}

In this first appendix section, we report the proofs for the theorem presented in the main body of the work. We make sure of the cleaner notation $W(P, Q)$ for a $p$-Wasserstein distance.

\begin{theorem}[Restatement of Theorem~\ref{thm:consist}]
    Let $\mathcal{D}_m$ be a test dataset. The test based on $T^{wass}_m = m^{1/2} W_p (P_{\theta}, Q)$ for hypotheses $H_0: D_m \sim P_{\theta}$ ~vs~ $H_1: D_m \sim Q \neq P_{\theta}$, is such that
    $$ \varrho (T^{wass}_m) = p(T^{wass}_m > \lambda_{1-\alpha} \mid Q \neq P_{\theta} ) \rightarrow 1 ~ , $$
    as $m \rightarrow \infty$, over alternatives $Q_m$ that satisfy $n^{1/2} W(P_{\theta}, Q_m) \geq \Delta_m$, where $\lim_{m \rightarrow \infty} \Delta_m = \infty$.
\end{theorem}

\begin{proof}
    Let $Q_m$ be a sequence of probabilities depending on $m$ satisfying $n^{1/2} W(P_{\theta}, Q_m) \geq \Delta_m$. Then, letting $\hat{Q}_m$ be an estimate of $Q_m$, by the triangle inequality of Wasserstein distance we have:
    $$ W(Q_m, P_{\theta}) \leq W(Q_m, \hat{Q}_m) + W(\hat{Q}_m , P_{\theta})  $$
    which implies:
    $$ T^{wass}_m = m^{1/2} W (P_{\theta}, Q) \geq \Delta_m - m^{1/2} W (Q_m, \hat{Q}_m) . $$
    Due to continuity of the cumulative distribution function relative to the random variable $m^{1/2} W (\cdot, \cdot)$, $F_W( \cdot)$, we can write $F_{m^{1/2} W}n (\lambda_{m, 1-\alpha}) \geq F_{\Delta - W} (\lambda_{1-\alpha})$, which can be written as
    { \small $$ p(T^{wass}_m \geq \lambda_{m, 1-\alpha}) \geq p(m^{1/2} W (Q_m, \hat{Q}_m) \leq \Delta_m - \lambda_{m, 1-\alpha}))$$}
    Now, we have that $m^{1/2} W (Q_m, \hat{Q}_m)$ is a tight sequence, i.e., $p(m^{1/2} W (Q_m, \hat{Q}_m) > M ) < \epsilon$ with $M > 0$, and $\epsilon > 0$ arbitrarily small. Thus we have that
    $$  p(m^{1/2} W (Q_m, \hat{Q}_m) \leq \Delta_m - \lambda_{m, 1-\alpha})) \rightarrow 1 $$
    implying $p(T^{wass}_m \geq \lambda_{m, 1-\alpha}) \rightarrow 1$ at the same time, since $\lim_{m \rightarrow \infty} \Delta_m = \infty$ and $\lambda_{m, 1-\alpha} \rightarrow \lambda_{1-\alpha}$, so that $\Delta_m - \lambda_{m, 1-\alpha} \rightarrow \infty$.
\end{proof}

We proceed with the proof of non-asymptotic lower bounds of Theorem 3.2.

\begin{theorem}[Restatement of Theorem~\ref{thm:3.2}]
    Let $\mathcal{D}_m$ be a test dataset. The test $T^{wass}_m = m^{1/2} W_p (P_{\theta}, Q)$ for hypotheses $H_0: D_m \sim P_{\theta}$ ~vs~ $H_1: D_m \sim Q \neq P_{\theta}$, is such that
    \begin{align*}
        \varrho (T^{wass}_m) & = p(T^{wass}_m > \lambda_{1-\alpha} \mid Q \neq P_{\theta} ) \geq \\
        & \geq 1 - \exp \left\{  - \gamma_p \frac{\phi'}{2} \big( \Delta_m - \lambda_{n, 1- \alpha} \big)^2  \right\} ~,
    \end{align*}
    if $m^{1/2} W_p (P_{\theta} , Q_m ) \geq \Delta_m$ and $\Delta_m \geq \lambda_{m, 1 - \alpha}$.
\end{theorem}

\begin{proof}
    Starting again from the triangle inequality property of Wasserstein distance as in the previous proof, we can obtain
    $$ T^{wass}_m = m^{1/2} W (P_{\theta}, Q) \geq \Delta_m - m^{1/2} W (Q_m, \hat{Q}_m) . $$
    Again, in the same way of the previous proof we can apply the cumulative distribution function $F_W( \cdot)$ and obtain 
    { \small $$ p(T^{wass}_m \geq \lambda_{m, 1-\alpha}) \geq p(m^{1/2} W (Q_m, \hat{Q}_m) \leq \Delta_m - \lambda_{m, 1-\alpha}))$$}
    Now, we make use of Wasserstein Concentration inequality theorem stated in \cite{bolley2007quantitative} that reads
    \begin{theorem}[\cite{bolley2007quantitative}] \label{thm:bolley}
        Let $p \in [1, 2]$ and let $P$ be a probability on $\mathbb{R}^d$ satisfying
        $$ W_p (P, Q) \leq \sqrt{\frac{2}{\phi} H (P | Q)}  = \sqrt{\frac{2}{\phi} \int \frac{dP}{dQ} \log \frac{dP}{dQ} dP} .$$
        For $d' > d$ and $\phi' < \phi$, there exists some constant $N_0$ such that for $N \geq N_0 \max (\epsilon^{-d' + 2} , 1)$ we have
        $$ p \big( W_p (P, \hat{P}_n \big) > \epsilon ) \leq \exp{ \left\{ - \gamma_p \frac{\phi'}{2} N \epsilon^2 \right\} } $$
        where:
        \begin{equation*}
        \gamma_p = 
            \begin{cases}
                1 \quad & \hspace{-0.25cm} \text{if } ~~ p \in [1,2) \\
                3 - 2\sqrt{2} \quad & \hspace{-0.25cm} \text{if } ~~ p=2
            \end{cases},
    \end{equation*}
    \end{theorem}

Applying the concentration inequality result of \cite{bolley2007quantitative} above to our can we can derive the following
\begin{align*}
    \hspace{-0.2cm} p(T^{wass}_m & \geq \lambda_{m, 1-\alpha}) \geq 
    \\
    & \geq p(m^{1/2} W (Q_m, \hat{Q}_m) \leq \Delta_m - \lambda_{m, 1-\alpha}))
    \\
    & \geq 1 - p(m^{1/2} W (Q_m, \hat{Q}_m) \geq \Delta_m - \lambda_{m, 1-\alpha}))
    \\
    & \geq 1 - p(W (Q_m, \hat{Q}_m) \geq m^{-1/2} (\Delta_m - \lambda_{m, 1-\alpha}))
    \\
    & \geq 1 - \exp{ \left\{ - \gamma_p \frac{\phi'}{2} m ( m^{-1/2} (\Delta_m - \lambda_{m, 1-\alpha}))^2 \right\} }
    \\
        & \geq 1 - \exp{ \left\{ - \gamma_p \frac{\phi'}{2} (\Delta_m - \lambda_{m, 1-\alpha})^2 \right\} } .
\end{align*}
\end{proof}
We proceed with the proof of Thm.~\ref{thm:wcupper} for the worst-case scenario where the two distributions tend to $m^{1/2} W_p (P_{\theta}, Q_m) \rightarrow 0$.

\begin{theorem}[Restatement of Theorem~\ref{thm:wcupper}]
    Let $\mathcal{D}_m$ be a test dataset. The test $T^{wass}_m = m^{1/2} W_p (P_{\theta}, Q)$ for hypotheses $H_0: D_m \sim P_{\theta}$ ~vs~ $H_1: D_m \sim Q \neq P_{\theta}$, is such that
    $$ \sup \varrho (T^{wass}_m) = p(T^{wass}_m > \lambda_{1-\alpha} \mid Q \neq P_{\theta} ) \leq \alpha $$
    as $m \rightarrow \infty$, for alternatives $Q_m$ satisfying $m^{1/2} W_p (P_{\theta}, Q_m) \rightarrow 0$.
\end{theorem}

\begin{proof}
    Starting from the Wasserstein triangle inequality:
    { \small $$ m^{1/2} W(Q_m, P_{\theta}) \leq m^{1/2} W(Q_m, \hat{Q}_m) + m^{1/2} W(\hat{Q}_m , P_{\theta})  $$}
    We can apply the cumulative distribution function to get $F_{m^{1/2} W}n (\lambda_{m, 1-\alpha}) \geq F_{m^{1/2} W + m^{1/2} W} (\lambda_{1-\alpha})$, that we can write as
    \begin{align*}
    p( & T^{wass}_m \geq \lambda_{m, 1-\alpha}) \geq 
    \\
    & \geq p( m^{1/2} W (P_{\theta}, Q_m) + m^{1/2} W (Q_m, \hat{Q}_m) \geq \lambda_{m, 1-\alpha})
    \\
    & \geq \sup p( m^{1/2} W (P_{\theta}, Q_m) \geq \lambda_{m, 1-\alpha})
    \end{align*}
    And then condition $m^{1/2} W_p (P_{\theta}, Q_m) \rightarrow 0$ trivially implies that in reality $Q=P_{\theta}$ ($H_0$ is true) and thus, by construction of the test, we have that: test power for accepting $H_1$ against alternative $Q_m$ $\overset{n}{\rightarrow}$ type-I error probability $= \alpha$. Which translate into:
    $$ \sup \varrho (T^{wass}_m) = p(T^{wass}_m > \lambda_{1-\alpha} \mid Q \neq P_{\theta} ) \leq \alpha $$
\end{proof}

Eventually, we conclude by proving Thm~\ref{thm:interbound} for the intermediate case where $m^{1/2} W_p (P_{\theta}, Q_m ) \rightarrow \delta$, such as the one depicted in the MNIST example.

\begin{theorem}[Restatement of Theorem~\ref{thm:interbound}] 
    Let $\mathcal{D}_m$ be a test dataset. The test $T^{wass}_m = m^{1/2} W_p (P_{\theta}, Q)$ for hypotheses $H_0: D_m \sim P_{\theta}$ ~vs~ $H_1: D_m \sim Q \neq P_{\theta}$, is such that
    \begin{align*}
        \varrho (T^{wass}_m) & = p(T^{wass}_m > \lambda_{1-\alpha} \mid Q \neq P_{\theta} ) \leq \\
        & \leq \exp \left\{  - \gamma_p \frac{\phi'}{2} \big( \lambda_{n, 1- \alpha} - \delta \big)^2  \right\} ~ ,
    \end{align*}
    for alternatives $Q_m$ such that $m^{1/2} W_p (P_{\theta}, Q_m ) \rightarrow \delta < \lambda_{n, 1 - \alpha}$.
\end{theorem}

\begin{proof}
    From the usual Wasserstein triangle inequality, plus the application of cumulative distribution function, we can derive similarly to other proofs:
    { \scriptsize $$ p(T^{wass}_m \geq \lambda_{m, 1-\alpha}) \leq p( W (Q_m, \hat{Q}_m) \leq m^{1/2} \lambda_{m, 1-\alpha} - W(P_{\theta}, Q_m))$$}
    Now applying the concentration inequality of Thm~\ref{thm:bolley} in \cite{bolley2007quantitative}, we can derive the following
    \begin{align*}
        p(T^{wass}_m & \geq \lambda_{m, 1-\alpha}) 
        \\
        & \leq \exp{\left\{ 
        - \gamma_p \frac{\phi'}{2} m \big(m^{-1/2} \lambda_{m, 1-\alpha} - W(P_{\theta}, Q)\big)^2 \right\}}
        \\
        & \leq \exp{\left\{ - \gamma_p \frac{\phi'}{2} ( \lambda_{m, 1-\alpha} - \delta )^2 \right\}}
    \end{align*}
\end{proof}

\section{Additional Information on Distances}

Firstly, we report here the definitions of some of the distances mentioned throughout the body of the paper. Let $\mathcal{X}$ be a compact metric set and let $P,Q \in \mathcal{P}_x$ be probability distributions defined on $\mathcal{X}$, then we can define the following distribution divergences/distances:
\begin{itemize}[leftmargin=2.5ex]
    \item The Kolmogorov-Smirnov (KS) Distance:
    $$ KS(P, Q) = \sup_x | F_P (x) - F_Q (x)  | $$
    where $F_i (x)$ is the cumulative distribution function evaluated at $x\in\mathcal{X}$
    \item The Kullback-Leibler (KL) Divergence:
    $$ KL(P, Q) = \int_{\mathcal{X}} p(x) \log \left( \frac{p(x)}{q(x)} \right) d x $$
    where $p(\cdot)$ and $q(\cdot)$ are density functions
    \item The Jensen-Shannon (JS) Divergence:
    $$ JS(P,Q) = KL(P, M) + KL(M, Q) ~ , ~ M = \frac{P + Q}{2} $$
    where $M$ is a mixture probability measure
    \item The Wasserstein (W) Distance:
    $$ W_p(P, Q)=\inf _{\gamma \in \Gamma(P,Q)} \Big( \mathbb{E}_{(Z,V) \sim \gamma} \| Z - V \|^p \Big)^{1 / p} $$
    where $\Gamma(P,Q)$ is the space of all couplings of $P$ and $Q$ (joint probability measures whose marginals are $P$ and $Q$).
\end{itemize}

\subsection{Link between KL and Wasserstein Distance}  \label{sec:appen:KLwass}

In this Appendix subsection we briefly sketch the mathematical relationship between the KL divergence and the Wasserstein distance. Following results in \cite{cai2022distances} we have that:
$$ d_{TV} (P, Q) \leq \sqrt{0.5 ~ d_{KL} (P,Q)} $$
where $d_{TV} (P, Q)$ is the Total Variation (TV) distance $d_{TV} (P,Q) = \sup_{x \in \mathcal{X}} |P(x) - Q(x)|$. While on the other hand then we have that 
$$ 2 W_1 (P,Q) \leq C d_{TV} (P,Q)  .$$
Thus putting the two together:
$$ \frac{2}{C} W_1 (P,Q) \leq \sqrt{0.5 ~ d_{KL} (P,Q)} . $$

\subsection{Simplifications for Gaussian Densities} \label{sec:simpli}

We conclude by including simplification of KL divergence and Wasserstein distance in the case of $P$ and $Q$ being (multivariate) Gaussian distributions, $P \sim \mathcal{N} (\mu_0, \Sigma_0)$ and $Q \sim \mathcal{N} (\mu_1, \Sigma_1)$. For the KL divergence we have 
\begin{align*} 
D_{\mathrm{KL}}\left(P, Q\right) = \frac{1}{2}\Big( & \operatorname{tr}\left(\Sigma_1^{-1} \Sigma_0\right)+\left(\mu_1-\mu_0\right)^{\top} \Sigma_1^{-1}\left(\mu_1-\mu_0\right) 
\\
& -k+ \ln \left(\frac{\operatorname{det} \Sigma_1}{\operatorname{det} \Sigma_0}\right) \Big).
\end{align*} 
While for the Wasserstein distance case we have:
\begin{align*}
    W_2\left(P, Q \right)^2 & = \left\|\mu_1-\mu_2\right\|_2^2+\operatorname{tr}\Big(\Sigma_1+ \Sigma_2-
    \\
    & - 2\left(\Sigma_2^{1 / 2} \Sigma_1 \Sigma_2^{1 / 2}\right)^{1 / 2}\Big).
\end{align*}

\end{document}